\documentclass[sn-mathphys]{sn-jnl}

\jyear{2022}%

\theoremstyle{thmstyleone}%
\newtheorem{theorem}{Theorem}
\newtheorem{proposition}[theorem]{Proposition}%

\theoremstyle{thmstyletwo}%

\theoremstyle{thmstylethree}%

\begin{document}

\title[CavT for Engagement Intensity Prediction]{Class-attention Video Transformer for Engagement Intensity Prediction}


\author*[1]{\fnm{Xusheng} \sur{Ai}}\email{00754@siit.edu.cn}

\author*[2]{\fnm{Victor S.} \sur{Sheng}}\email{victor.sheng@ttu.edu}

\author[1]{\fnm{Chunhua} \sur{Li}}\email{00804@siit.edu.cn}
\author[3]{\fnm{Zhiming} \sur{Cui}}\email{zmcui@usts.edu.cn}

\affil*[1]{\orgdiv{Software and Service Outsourcing College}, \orgname{Suzhou Vocational Institute of Industrial Technology}, \orgaddress{\city{Suzhou}, \postcode{215104}, \country{China}}}

\affil[2]{\orgdiv{Department of Computer Science}, \orgname{Texas Tech University}, \orgaddress{\city{Lubbock}, \postcode{79409}, \country{USA}}}

\affil[3]{\orgdiv{School of Electronics and Information Engineering}, \orgname{Suzhou University of Science and Technology}, \orgaddress{\city{Suzhou}, \postcode{215009},  \country{China}}}


\abstract{In order to deal with variant-length long videos, prior works extract multi-modal features and fuse them to predict students' engagement intensity. In this paper, we present a new end-to-end method Class Attention in Video Transformer (CavT), which involves a single vector to process class embedding and to uniformly perform end-to-end learning on variant-length long videos and fixed-length short videos. Furthermore, to address the lack of sufficient samples, we propose a binary-order representatives sampling method (BorS) to add multiple video sequences of each video to augment the training set. BorS+CavT not only achieves the state-of-the-art MSE (0.0495) on the EmotiW-EP dataset, but also obtains the state-of-the-art MSE (0.0377) on the DAiSEE dataset. The code and models have been made publicly available at https://github.com/mountainai/cavt.
}




\maketitle

\section{Introduction}
\label{sec:intro}
With the rise of the internet, online education attracts increasing attention. Intelligent tutoring systems (ITS) and massively open online courses (MOOCs) have become fashionable worldwide. Students gradually get used to remotely studying all the materials they are interested in and obtaining digital certificates. Additionally, COVID-19 pandemic accelerates this tendency. Many universities, high schools, elementary schools, and other educational institutions have been shut down, and teachers have to undertake teaching remotely.

However, a new problem emerges. Even though schools have made the complete transition from traditional in-person to online teaching, teachers are struggling to engage students.  In an offline classroom, a teacher can access to the student engagement intensity through observing facial expressions and social cues, such as yawning, glued eyes, and body postures. But it seems an impossible task for teachers to assess the student engagement in an online learning system, as students could participate in online education using various types of electronic devices under different backgrounds, e.g., home, yards, and parks. Research indicates that many students, estimated between 25\% and 60\%, are reported being disengaged or barely engaged to online courses \cite{nicholls2015some}. Therefore, automatic engagement prediction is of great importance to quantify online educational quality for teachers and learning efficiency for students.

There are various ways to automatically detect engagement, such as action recognition, sensor-based methods, and computer vision methods \cite{whitehill2014faces}. Among these approaches, computer-vision methods based on face extraction have received the most attention. There are two categories about automatic engagement detection. One category is based on handcrafted features and the other makes use of end-to-end models. The former extracts multi-model features, such as eye gaze, head pose, facial expression, and facial features for models to predict the levels of engagement \cite{kaur2018prediction,thong2019engagement,zhu2020multi}. The latter directly feeds video frames to Convolutional Neural Networks (CNNs) to detect the level of engagement \cite{gupta2016daisee,huang2019fine,wang2020automated}.

From the movement in the face (eye gaze, head pose) as well as the facial expression to the facial features obtained by CNN models, multi-model feature extraction is often a tedious and unpleasant task. Moreover, it is uneasy to design a model to fuse multi-modal features. Practitioners \cite{kaur2018prediction,thong2019engagement,zhu2020multi} have to design complicated nested ensembles or hybrid networks to fuse different features extracted from video such that it is difficult to adapt their methods to another database (or scenario).  Additionally, class imbalance is widely seen in many engagement databases, i.e., the EmotiW-EP dataset. For instance, \cite{zhu2020multi} tries to solve the class imbalance problem by applying weighted loss, but they still fail to predict engagement intensity for minority classes. These factors inspire us to develop end-to-end models. 

In the meanwhile, some researchers proposed end-to-end methods to fulfill engagement prediction tasks. For instance, CNN-based ResNet-TCN consists of ResNet and TCN \cite{abedi2021improving}. ResNet extracts spatial features from consecutive video frames whereas TCN captures temporal changes to detect the level of engagement. But in a large-scale image classification, attention-based transformers have been recently shaking up the long supremacy of CNNs. This drives us to compare transformers with CNNs for engagement intensity prediction. Furthermore, the optimization of representative frame selection has been little studied so far.  Limited by computation and space capacity, end-to-end methods usually select representative frames from video and discard the other frames. This leads to a low frame utilization ratio, especially for long videos. As a result, these end-to-end methods are prone to yield a model on which there is a large gap between the training set and the validation set (i.e., low generalization). 

Aiming at transformer adaptation and high frame utilization, we propose a transformer-based method, Class Attention in Video Transformer (CavT), which directly takes a sequence of video frames as the effective input and outputs the engagement intensity of the video in this paper. To make use of as many frames as possible, we also present Binary Order Representatives Sampling (BorS), which adaptively divides different number of frames into sliding windows, and selects representatives each window, and generates video sequences in each video. Consequently, CavT combined with BorS can significantly improve the accuracy of engagement intensity prediction for diverse videos, especially for long videos. The main contributions of this paper are as follows:
\begin{enumerate}
\item We propose a transformer-based CavT method, which combines self-attention between patches and class-attention between class token and patches, to predict the engagement intensity. In contrast to the state-of-the-art methods, CavT can fuse facial spatiotemporal information, which is favorable for capturing the fine-grained engaged state and improving the engagement prediction performance.
\item To address the low-frame-utilization problem, we use BorS to generate multiple video sequences in each video and add them to the training set. Our experimental results prove that BorS is helpful in perceiving the temporal changes of different videos obtained from diverse conditions. 
\item The proposed approach shows strong performance on both the EmotiW-EP dataset and DAiSEE dataset. BorS+CavT achieves 0.0495 MSE on Emotiw-EP and 0.0377 on DAiSEE respectively, which are much better than that of the state-of-the-art methods.
\end{enumerate}

The rest of the paper is organized as follows. In Section~\ref{sec:related}, we introduce related work. In Section~\ref{sec:ourmethod}, we describe a state-of-the-art video sequence generation method and a video transformer for engagement intensity prediction. In Section~\ref{sec:experiments}, we show our experimental results on the EmotiW-EP dataset and DAiSEE dataset respectively. Finally, we conclude our study in Section~\ref{sec:conclusion}.

\section{Related work}
\label{sec:related}
Computer vision has come into use for automatic engagement measurement over the past decade \cite{mohamad2019automatic}. At an early stage, Grafsgaard et al. \cite{grafsgaard2013automatically} analyzed a video corpus of computer-mediated human tutoring based on the Computer Expression Recognition Toolbox (CERT) \cite{littlewort2011computer}. The results of this research reveal significant relationships among facial expression, frustration, and learning. Bosch et al. \cite{bosch2015automatic} adopted computer vision and machine learning techniques to detect five affective states (boredom, confusion, delight, engagement, and frustration). Saneiro et al. \cite{saneiro2014towards} analyzed coherent connections between students' affective states and their multimodal features (facial expression and body movements). These multimodal features combined with additional affective information altogether can serve as automatic engagement prediction. But these methods do not clearly formulate engagement levels and are not designed for wild conditions. 

In the following, Whitehill et al. \cite{whitehill2014faces} developed automatic engagement detectors based on facial expressions and machine learning techniques. They found that automated engagement detectors perform with comparable accuracy to humans for distinguishing high engagement from low engagement. Kamath et al. \cite{kamath2016crowdsourced} proposed an instance-weighted Multiple Kernel Learning SVM method to improve the accuracy of engagement prediction. Monkaresi et al. \cite{monkaresi2016automated} extracted three sets of features from videos, heart rates, Animation Units (from Microsoft Kinect Face Tracker), and local binary patterns in three orthogonal planes (LBP-TOP) \cite{whitehill2014faces}. These features are used in machine learning for training student-independent engagement detectors. However, it is not easy to accurately acquire these hand-crafted features and perform feature engineering on hundreds or even thousands of raw features. 

With the rapid development of deep learning techniques in the computer vision field, they have increasingly been used for engagement prediction. Researchers start considering diverse conditions, i.e., insufficient illumination and strong reflections. Gupta et al. \cite{gupta2016daisee} built the "in the wild" video corpus DAiSEE labelled with four engagement levels. Each video lasts 10 seconds and was recorded in different places. The organizers of the EmotiW Challenges publicized a video database EmotiW-EP in diverse and "in the wild" conditions \cite{kaur2018prediction}. The duration of the videos ranges from $\sim$3 minutes to $\sim$5 minutes. We examined recent works on the two dataset and categorized computer vision based approaches into two groups, i.e., feature-based models and end-to-end models. 

In the feature-based approaches, Yang et al. \cite{yang2018deep} divided a video into segments and calculated the temporal-spatial features of each segment for regressing engagement intensity. They built a multi-modal regression model based on a multi-instance learning framework with a Long Short Term Memory (LSTM), and conducted model ensemble on multiple data splits. Wu et al. \cite{wu2020advanced} extracted motion features, such as facial movements, upper-body posture movements and overall environmental movements in a given time interval, and then incorporated these motion features into a deep learning model consisting of LSTM, Gated Recurrent Unit (GRU) and a fully connected layer. Huang et al. \cite{huang2019fine} firstly captured multi-modal features from faces by using the OpenFace library \cite{baltrusaitis2018openface}, and proposed a Deep Engagement Recognition Network (DERN) combining Bi-LSTM with attention mechanisms to classify the level of engagement on the DAiSEE dataset. Wang et al. \cite{wang2020automated} proposed a CNN architecture to predict the level of engagement on the DAiSEE dataset based on facial landmarks and features extracting from faces. However, the feature extraction process may loss some valuable information more or less such that the lack of key features ruins robustness. Furthermore, some existing methods leverage ensemble models to fuse different features. Ensemble models work well only if the base learners are good and heterogeneous. Ensemble models often fail to work on different databases for their complexity. This motivates us to seek for end-to-end approaches to improve performance and robustness.

In the end-to-end approaches, a deep learning classifier or regressor feeds raw frames of videos or images to predict the level of engagement (for classifiers) or the engagement intensity (for regressors). Geng et al. \cite{geng2019learning} utilized a Convolutional 3D (C3D) model for automatic engagement prediction whereas the focal loss \cite{lin2017focal} is introduced to counter the class-imbalance problem by adaptively adjusting the weights of different levels of engagement. Zhang et al. \cite{zhang2019novel} adapted Inflated 3D Convolutional Network (I3D) to the field of automatic students' engagement recognition. To tackle the imbalanced data distribution of labels of DAiSEE videos, the focal loss and weighted loss \cite{he2019metanet} are used to prevent from misclassifying not-engaged samples. Liao et el. \cite{liao2021deep} presented a Deep Facial Spatio-Temporal Network (DFSTN) for students’ engagement detection for online learning. DFSTN includes a pre-trained SE-ResNet-50 for extracting spatial features from faces and a LSTM with global attention for generating an attentional hidden state. To the best of our knowledge, it is the first time to investigate the performance on both the DAiSEE dataset and the EmotiW-EP dataset. However, DFSTN does not take into account video length. Instead, it samples 20 frames in each video. Actually, the DAiSEE video is 10-second long whereas the duration of the EmotiW-EP video averages between 3 and 5 minutes. Fixed-size video sequence cannot take fully advantage of video frames. Furthermore, a modeling shift is on the way for image recognition and video classification from CNNs \cite{he2016deep,xie2018rethinking} to Transformers \cite{wu2021cvt,lee2021vision,liu2022video}. Therefore, based on the Class Attention in Image Transformer \cite{touvron2021going}, we present Class Attention in Video Transform (CavT) to build a video transformer for engagement intensity prediction. 

\section{Our method}
\label{sec:ourmethod}
In this section, we first introduce the overall architecture of BorS+CavT to address the engagement intensity prediction problem. Then, we  describe Binary Order Representatives Sampling (BorS) to generate heterogeneous video sequences for each video, and use it with CavT to improve the performance of engagement intensity prediction. 

\subsection{Overall architecture}
The overall architecture of our proposed method BorS+CavT is shown in Fig.~\ref{fig:ourarch}. Firstly, the input video is preprocessed by face extraction. Face is the most expressive region reflecting a person's emotional state. Face extraction makes our architecture more robust to challenging conditions, such as low illumination and strong reflection. We leverage OpenFace 2.0 \cite{baltrusaitis2018openface} to extract face images from raw frames serving as data preprocessing. In the following, Binary-Order Representatives Sampling (BorS) yields $r$ sub videos (video sequences) defined to be of size $T \times H \times W \times 3$, consisting of $T$ frames and each frame containing $H \times W \times 3$ pixels, where ($H$, $W$) is the resolution of the face image and 3 corresponds to RGB channels. 

Then, Class Attention in Video Transformer (CavT) transforms a video sequence into a vector. At the beginning, 3D Patch Partition reshapes the video sequence $x \in R^{T \times H \times W \times 3}$ into a sequence of flattened patches along temporal and spatial dimension, generating $k$ = $T/t \cdot H/p \cdot W/p$ patches with patch size = ($t$, $p$, $p$, 3), each of which contains $3tp^2$ pixels. The $k$ patches, [$x_{p}^{1},\dots, x_{p}^k$], serve as the effective input sequence. CavT uses a constant latent vector size $c$ through all layers, so we map $x_p^i$ to a $c$-dimension vector with a trainable linear projection \eqref{eq:projection}. We refer to the output of this projection, $x_e$,  as the patch embeddings.
\begin{equation}
x_{e}=[x_{p}^{1}E;x_{p}^{2}E;,...;x_{p}^{k}E],~~~E\in R^{(3tp^2\times c)}\label{eq:projection}
\end{equation}

\begin{figure}[h]%
\centering
\includegraphics[width=0.98\textwidth]{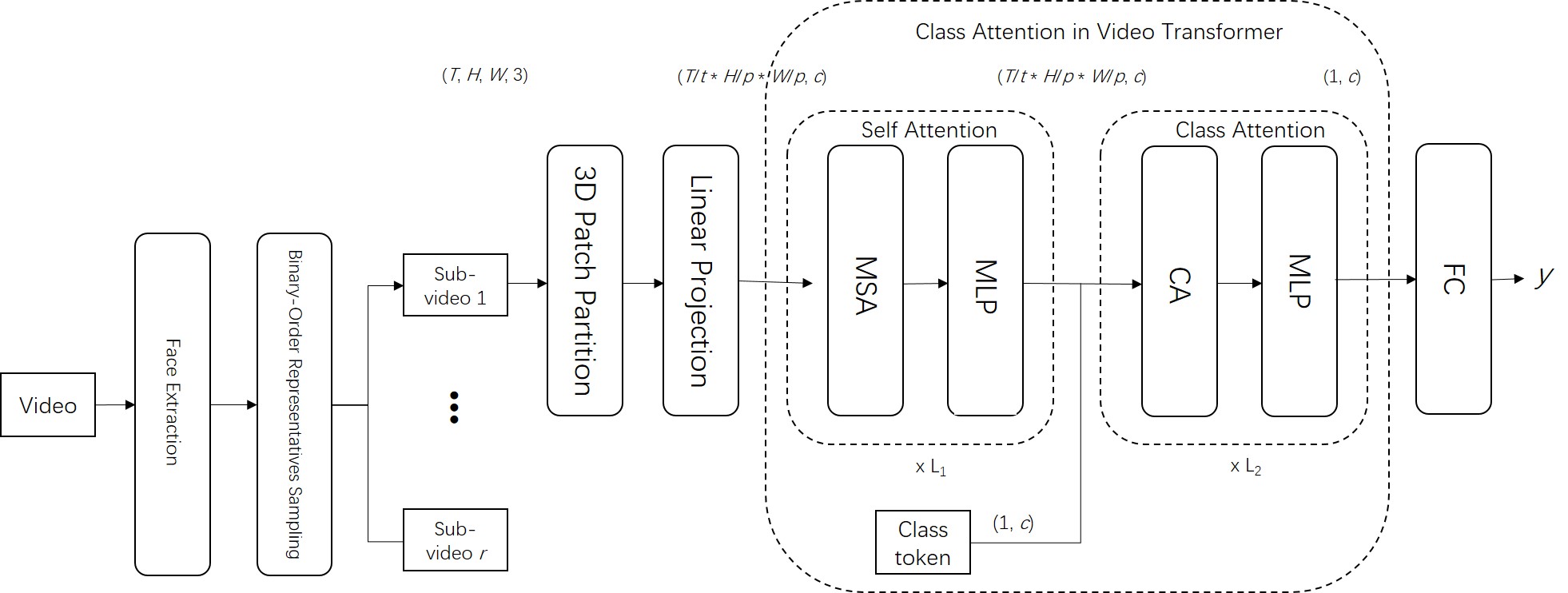}
\caption{The overall architecture of BorS+CavT.}\label{fig:ourarch}
\end{figure}

After that, the self-attention (SA) layers are devoted to compute self-attention weights between patches, whereas the class attention (CA) layers involve the patch embeddings into a simple vector (class embedding) so that it can be fed to a linear classifier. 

Self-attention layer consists of MSA and MLP module, which is identical to standard ViT \cite{dosovitskiy2020image}, but with no class token. Class embedding, which is initialized as class token (denoted by CLS), is inserted on the class-attention stage containing CA modules and MLP modules, and is updated by residual in CA and MLP processing of the class-attention stage. Normalization (LN) layer is applied before each MSA, CA, and MLP module. In the output layer, we use a full-connected layer (FC) to transform the class embedding into $y \in [0, 1]$ representing engagement intensity.

\subsection{Binary-order representative resampling}
To address the lack-of-sample problem, we design Binary-Order Representatives Sampling (BorS) to build multiple video sequences from each video. BorS includes three steps: 1) downsample video; 2) partition frames into slide windows; 3) select representative frame in each slide window. For the convenience of discussion, we define a few notations or functions as follows: \\
$n$: Number of frames. \\
$T$: Number of slide windows. \\
$r$: Number of video sequences to generate from a video. \\
$\gamma$: Sample rate to select the frames of slide windows. \\
$\alpha$: Ratio of the window size and the stride size. \\
$\left\lfloor \right\rfloor$: A function truncates a number to an integer by removing the fractional part of the number.

\subsubsection{Downsample video}
Successive frames are relatively invariant at an interval. To save space and time, we select frames at $\gamma$ intervals. If $\gamma$ is an aliquant part of $n$, we drop the last few frames and let $n$ be $\gamma$ $\left\lfloor n/\gamma \right\rfloor$. Then a video with $n$ frames are compressed into the compressed version with $n/\gamma$ frames.

\subsubsection{Partition frames into slide windows}
In the following, we divide $n/\gamma$ frames into $T$ slide windows. When $T$, $\gamma$, and $\alpha$ are hyper-parameters, we estimate the lower bound of $n$. 

\begin{proposition} \label{prop1}
If $n >= \gamma(T + \alpha - 1)$, $n$ frames can be divided into $T$ slide windows. 
\end{proposition}

\begin{proof}[Proof of Proposition~{\upshape\ref{prop1}}]
We define the window size $\zeta$ and the stride size $\xi$. The first window starts at 1 and ends at $\zeta$, the second window starts at $\xi$ + 1 and ends at $\xi$ + $\zeta$, and so on. Then $T$ is obtained as follows:

\begin{align}
  T  &=  {\frac{{n/\gamma  - \zeta }}{\xi } + 1}  \\
     &=  \frac{{n/\gamma  - \alpha \xi }}{\xi } + 1  \\
     &=  \frac{n}{{\gamma \xi }} - \alpha  + 1
\end{align}
or, equivalently,
\begin{align}
  \xi  &= \frac{n}{{\gamma (T + \alpha  - 1)}}
\end{align}
Since $\xi$ is a positive integer, $n >= \gamma(T + \alpha - 1)$ have to be satisfied. That is, if $n >= \gamma(T + \alpha - 1)$, $n$ frames can be divided into $T$ slide windows.
\end{proof}

Usually, $T$, $\gamma$, and $\alpha$ are set small numbers. So the pre-condition is easy to satisfy. Next, we will select different representative frames in each slide window to build $r$ video sequences. 

\subsubsection{Select representative frame in each slide window}
Based on the order of the frames in a slide window, we build a binary tree for each slide window. For instance, a binary tree is shown in Fig.~\ref{fig:bors} when the size of a slide window is 7.  Seen from Fig. 2, the $4^{th}$ frame is the root of the binary tree, the $2^{nd}$ frame is the root of the left subtree, the $6^{th}$ frame is the root of the right subtree, the $1^{st}$ frame is the leftmost leaf, …, the $7^{th}$ frame is located in the rightmost of the bottom layer.

\begin{figure}[h]%
\centering
\includegraphics[width=0.9\textwidth]{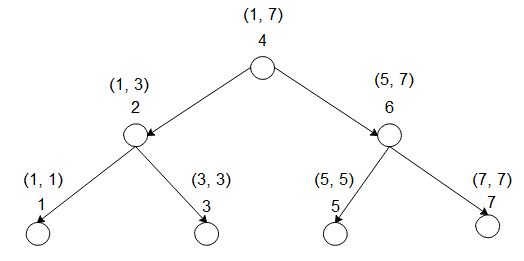}
\caption{An illustration of binary-order representatives of a window. In each node, (a, b) signifies that the window starts from a and ends at b, where a and b are the frame index. The number under window (a, b) is the index of the representative frame.}\label{fig:bors}
\end{figure}

After that, we select different representative frames in the breadth-first-search order of the binary tree in each window to make up $r$ video sequences. For instance, assuming (1 + $\zeta$) mod $2^r$ = 0, BorS yields $r$ video sequences from a video in order as follows. 
\begin{align*}
S^1 &= [(1 + \zeta)/2, (1 + \zeta)/2 + \xi, \cdots, (1 + \zeta)/2 + (T – 1)\xi]\\
S^2 &= [(1 + \zeta)/4, (1 + \zeta)/4 + \xi, \cdots, (1 + \zeta)/4 + (T – 1)\xi]\\
    & \dots\\
S^r &= [(1 + \zeta)/2^r, (1 + \zeta)/2^r + \xi, \cdots, (1 + \zeta)/2^r + (T – 1)\xi]
\end{align*}
 
Finally, the $r$ video sequences are added to the training data. When an end-to-end model fits the enriched training data at an epoch, the resulting model is used for predicting a test video by feeding the first video sequence of the test video $S^1$. As any element of a video sequence is located in the middle of a slide window or its descendant, neighboring frames have little chance of being together elected as representatives when $r$ is small. Thus, BorS can generate heterogeneous video sequences representing as many video fragments as possible. Consequently, our sampling method elaborately avoids both over-fitting caused by homogeneous video sequences to the train data and under-fitting introduced by synthetic samples.

\subsection{Class attention in video transformer}
A video sequence is divided into patches along with temporal and spatial dimension. Each patch contains the same parts of $t$ neighboring images. Patches are transformed into patch embeddings $x_e$ using a trainable linear projection \eqref{eq:projection}. Then, Class-attention Video Transformer (CavT), which is composed of self-attention stage and class-attention stage, compiles the patch embeddings into a $c$-dimension vector. The CavT architecture is depicted in Fig.~\ref{fig:cavt}.

\begin{figure}[h]%
\centering
\includegraphics[width=0.9\textwidth]{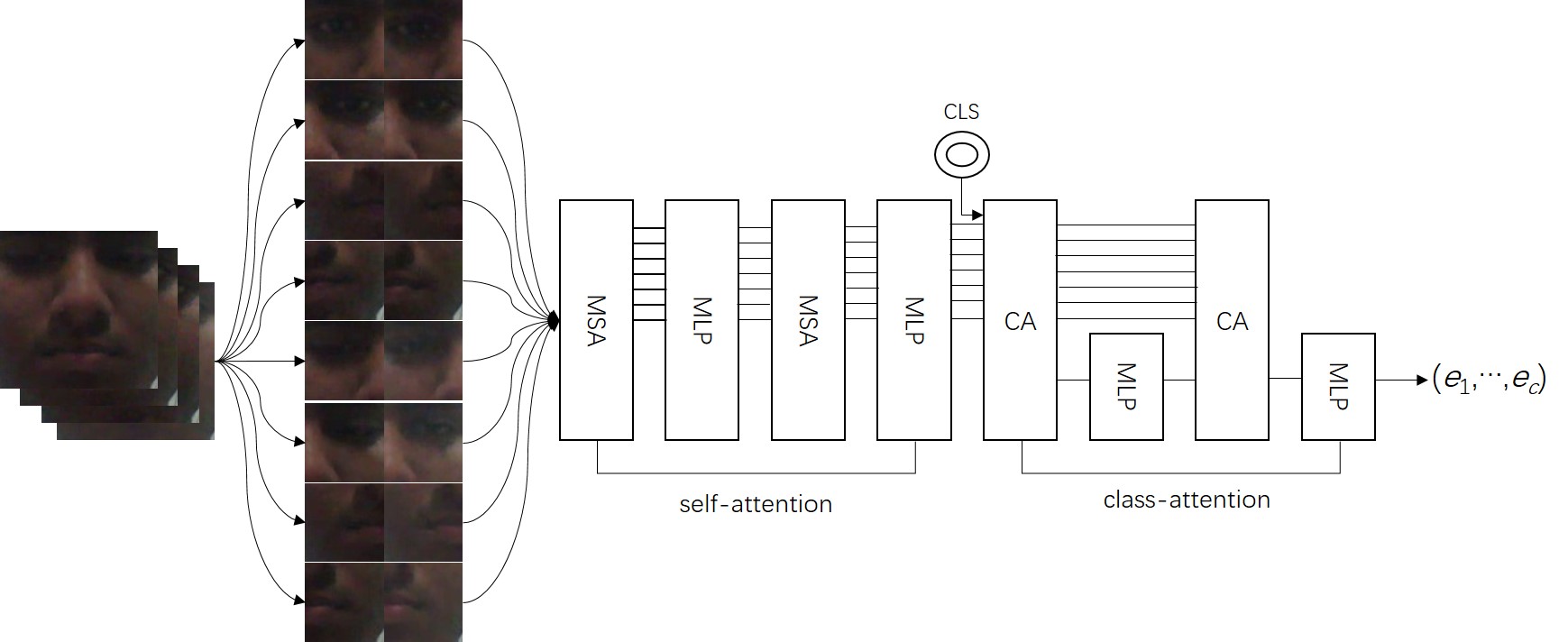}
\caption{The architecture of Class-Attention in Video Transformer (CavT).}\label{fig:cavt}
\end{figure}

\subsubsection{Self-attention stage}
Self-attention stage, which consists of MSA (multi-heads self-attention) and MLP (multi-layer perception) modules, guides the self-attention between patches. The definition of MSA, MLP and LN is identical to that of the ViT transformer \cite{dosovitskiy2020image}. Then we formulate self-attention stage by \eqref{eq:sa_1}, \eqref{eq:sa_2}, and \eqref{eq:sa_3}.
\begin{equation}
u_0=x_e\label{eq:sa_1}
\end{equation}
\begin{equation}
u'_l=diag(\lambda_{l-1,1},\dots,\lambda_{l-1,c})MSA(LN(u_{l-1}))+u_{l-1} \label{eq:sa_2}
\end{equation}
\begin{equation}
u_l=diag(\lambda_{l,1},\dots,\lambda_{l,c})MLP(LN(u'_l))+u'_l \label{eq:sa_3}
\end{equation}
, where the parameters $\lambda_{l-1,i}$, $\lambda_{l,i}$ are learnable weights and $l = 1,\dots,L_1$.

In \eqref{eq:sa_2} and \eqref{eq:sa_3}, we multiply the output of each residual block by a diagonal matrix. We aim at grouping the updates of the weights associated with the same output channel. At the initialization stage, the diagonal is initialized with small values. As training process proceeds, the diagonal values are gradually adaptive to the output of the residual blocks per channel.

\subsubsection{Class-attention stage}
Class-attention stage includes multiple CA layers, each of which consists of a CA module and a MLP module. Initially, a class token (CLS) $ \in R^c$ is inserted as the initial class embedding $x_{class}$. Then, $x_{class}$ plus $u_{L_1}$ (the patch embeddings outputted by the last SA layer) is updated by the CA module and then is updated by the MLP module in turn, as defined in \eqref{eq:ca_1}, \eqref{eq:ca_2}, and \eqref{eq:ca_3}.
\begin{equation}
z_0=[{\rm CLS}, u_{L_1}] \label{eq:ca_1}
\end{equation}
\begin{equation}
z'_l=diag(\beta_{l-1,1},\dots,\beta_{l-1,c})CA(LN(z_{l-1}))+z_{l-1} \label{eq:ca_2}
\end{equation}
\begin{equation}
z_l=diag(\beta_{l,1},\dots,\beta_{l,c})MLP(LN(z'_l))+z'_l \label{eq:ca_3}
\end{equation}
, where the parameters $\beta_{l-1,i}$, $\beta_{l,i}$ are learnable weights and $l = 1,\dots,L_2$.

The CA module is identical to a MSA module, except that it relies on the attention between the class embedding $x_{class}$ and itself plus the patch embeddings $x_u$. Assuming a network with $h$ heads and $k$ patches, we parameterize the CA module with projection matrices, $W_q, W_k, W_v, W_o \in R^{c \times c}$, and the corresponding biases $b_q, b_k, b_v, b_o \in R^c$. We perform the CA module as follows. 
\begin{equation}
Q = {W_q}LN({x_{class}}) + {b_q} \label{eq:ca_4}
\end{equation}
\begin{equation}
K = {W_k}LN([{x_{class}},{x_u}]) + {b_k} \label{eq:ca_5}
\end{equation}
\begin{equation}
V = {W_v}LN([{x_{class}},{x_u}]) + {b_v} \label{eq:ca_6}
\end{equation}

The class-attention weights are given by
\begin{equation}
A = \mathop{\rm Softmax}(QK^T / \sqrt {c/h}) \label{eq:ca_7}
\end{equation}
where $QK^T \in R^{h \times 1 \times k}$. This attention is involved in the weighted sum $A \times V$ to produce the residual output vector.
\begin{equation}
{{\mathop{\rm out}\nolimits} _{CA}} = W_oAV + b_o \label{eq:ca_8}
\end{equation}

The CA module extracts the useful information from the patch embeddings to the class embedding. Empirically, a set of 2 blocks of modules (2 CA and 2 MLP) are sufficient to achieve good performance. In Section \ref{sec:experiments} (Experiments), our CavT consists of 12 blocks of SA+MLP modules ($L_1$ = 12) and 2 blocks of CA+MLP modules ($L_2$ = 2).

\section{Experiments}
\label{sec:experiments}
In this section, we conduct experiments to investigate the performance of our method, comparing to the state-of-the-art methods.

\subsection{Setup}
\subsubsection{Datasets}
We adopt the widely-used EmotiW-EP \cite{zhu2020multi} and DAiSEE \cite{gupta2016daisee} for performance analysis. The EmotiW-EP dataset contains 146 videos in the training set and 48 videos in the validation set.  Observed from Table~\ref{tab:emotiw_dist}, there are only several videos labelled 0 (denoting very low engagement). The DAiSEE dataset consists of 5358 training videos, 1429 validation videos, and 1784 test videos in four engagement intensity levels. We combined the train and validation set to train models and report results on 1784 test videos. As it can be seen in Table~\ref{tab:daisee_dist}, the dataset is highly imbalanced, only 0.63\%, 1.61\%, and 0.22\% percent of train, validation, and test sets are in the engagement intensity level 0 (corresponding to very low engagement). Figure~\ref{fig:emotiw_samples} and Figure~\ref{fig:daisee_samples} illustrate the examples of videos from the EmotiwW-EP dataset and the DAiSEE dataset respectively.

\begin{table}[h]
\begin{center}
\begin{minipage}{250pt}
\caption{\# samples in train and validation set on EmotiW-EP.}\label{tab:emotiw_dist}%
\begin{tabular}{p{2.5cm}p{2.5cm}p{2.5cm}}
\toprule
Engagement intensity & \# of training videos & \# of validation videos \\
\midrule
0  &  5  &  4 \\ 
0.33  &  35  &  10 \\ 
0.66  &  81  &  19 \\ 
1  &  28  &  15 \\ 
\botrule
\end{tabular}
\end{minipage}
\end{center}
\end{table}

\begin{table}[h]
\begin{center}
\begin{minipage}{250pt}
\caption{\# samples in train, validation and test set on DAiSEE.}\label{tab:daisee_dist}%
\begin{tabular}{p{2cm}p{2cm}p{2cm}p{2cm}}
\toprule
Engagement intensity & \# of training videos & \# of validation videos & \# of test videos\\
\midrule
0  &  34  &  23  &  4\\ 
0.33  &  213  &  143  &  84\\ 
0.66  &  2617  &  813  &  882\\ 
1  &  2494  &  450  &  814\\ 
\botrule
\end{tabular}
\end{minipage}
\end{center}
\end{table}

\begin{figure}[h]%
\centering
\includegraphics[width=0.9\textwidth]{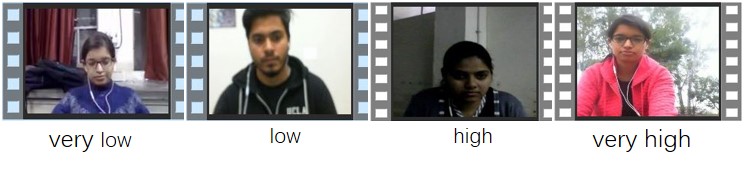}
\caption{Examples of videos from EmotiW-EP dataset. Left to right columns show engagement intensity level: [0 (very low) - 3 (very high)]. Videos have several FPS, i.e., 15, 30, 29.89. Duration of videos varies from 3 minutes to 6 minutes. \# of frames fluctuates from 1500 to 10000.}\label{fig:emotiw_samples}
\end{figure}

\begin{figure}[h]%
\centering
\includegraphics[width=0.9\textwidth]{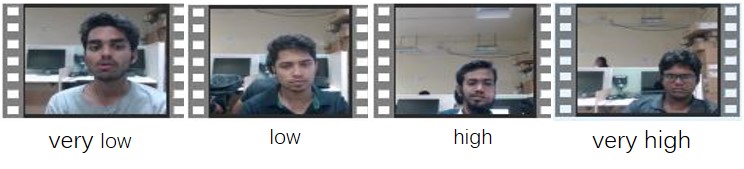}
\caption{Examples of videos from DAiSEE dataset. Left to right columns show engagement intensity level: [0 (very low) - 3 (very high)]. FPS of videos is fixed 30. Each video lasts 10 seconds so the frames in a video numbers 300 in all.}\label{fig:daisee_samples}
\end{figure}

\subsubsection{Implementation details}
We employ the Adam optimizer \cite{kingma2014adam} for 20 epochs with a learning rate 1e-5.  Following \cite{huang2016deep}, the stochastic depth is employed with drop rate 0.05. We empirically configure the resolution of the face image (112, 112) and the patch size (2, 14, 14, 3) respectively. We adopt 16 heads to compute self-attention weights of 1024-dimension patch embeddings or class embedding. When BorS is applied, we set the sampling rate 5 and the ratio of the window size to the stride size 3 for experiments on the EmotiW-EP dataset, whereas sampling rate = 1 and the ratio of the window size to the stride size = 5 for experiments on the DAiSEE dataset. We adopt Mini-batch Stochastic Gradient Descent (MBGD) to fully utilize the power of our H3C G4900 server, which is equipped with a 32G Tesla GPU card. We configure the batch size 4 for training on the EmotiW-EP dataset and 8 for training on the DAiSEE dataset, respectively. 

\subsection{Comparisons with state-of-the-art methods}
\subsubsection{EmotiW-EP}
For performance comparisons, we compare the performance of CavT with the end-to-end methods (ResNet+TCN \cite{abedi2021improving}, C3D \cite{geng2019learning}, I3D \cite{zhang2019novel}, Swin-L \cite{liu2022video}, S3D \cite{xie2018rethinking}) and feature-based method MAGRU \cite{zhu2020multi}. We also implemented the combination of the BorS and CavT (BorS+CavT) with sampling times $r$ = 4 to investigate its performance, comparing to CavT. Zhu et al. does not publish code with their paper so we directly excerpt their experimental results from \cite{zhu2020multi}. We implemented the ResNet+TCN, C3D, I3D, Swin-L, S3D method based on several Github public repositories. 

Table~\ref{tab:emotiw-method-comp} shows the experimental results of different methods on the EmotiW-EP dataset. In this table, the MSE and mean of individual MSEs (MMSE) are shown for our methods and prior works. Table III indicates that the CavT method achieves the lower MSE (0.0667) and MMSE (0.0806) than all prior end-to-end methods. The BorS+CavT method achieves the lowest MSE (0.0495) and MMSE (0.0673) among all of methods, including the state-of-the-art MAGRU. Zhu et al. does not report the MMSE result of the MAGRU method on the validation set. But based on their results on the test set (0.105), we infer that their MMSE on the validation set should be still high. MMSE is a useful indicator measuring the degree of class imbalance. When we train an end-to-end model on an imbalanced video dataset, we still can generate multiple video sequences for each video to decrease the MSE of individual engagement level such that a low MMSE as well as a low MSE is attained. Furthermore, the EmotiW-EP videos are minute-long. It is eligible to build multiple heterogeneous video sequences for each video to improve model generalization. This is why BorS+CavT achieves the lower MSE than CavT does, but without applying any balancing method as \cite{zhu2020multi} does. 

The fourth column of Table~\ref{tab:emotiw-method-comp} shows the number of frames in a video sequence. Compared to prior end-to-end methods, we configure CavT with the least number of frames (32), but it can achieve the lowest MSE. The fifth and sixth column of this table shows computation complexity in term of FLOPs and Parameter. Compared to the state-of-the-art video transformer Swin-L, the cost of CavT is less expensive in term of computation and memory. Additionally, BorS only increases computation cost on the training stage, so that FlOPs and Parameters of the BorS+CavT method are identical to those of the CavT method.

\begin{table}[h]
\begin{center}
\begin{minipage}{280pt}
\caption{Performance comparisons on EmotiW-EP. "MMSE" indicates the mean of MSEs of very low (0), low (0.33), high (0.66) and very high (1). "Frames" indicates \# of frames in a video sequence. The magnitudes are GIGA (10e9) and MEGA(10e6) for FLOPs and Parameter respectively.}\label{tab:emotiw-method-comp}%
\begin{tabular}{p{2cm}p{1.1cm}p{1.1cm}p{1.1cm}p{1.1cm}p{1.1cm}}
\toprule
Symbol & MSE & MMSE & Frames & FLOPs & Param\\
\midrule
CavT  &  0.0667  &  0.0806  &  32  &  89.33  &  119.85\\ 
BorS+CavT  &  \pmb{0.0495}  &  \pmb{0.0673}  &  32  &  89.33  &  119.85\\ 
ResNet+TCN  &  0.096  &  0.1913  &  50  &  24.78  &  12.95\\ 
C3D  &  0.0904  &  0.1635  &  128  &  270.19  &  63.21\\ 
I3D  &  0.0741  &  0.0882  &  64  &  111.33  &  14.39\\ 
Swin-L  &  0.0813  &  0.129  &  32  &  545.1  &  195.58\\ 
S3D  &  0.1309  &  0.1406  &  64  &  72.69  &  10.02\\ 
MAGRU  &  0.0517  &  n/a  &  n/a  &  n/a  &  n/a\\ 
\botrule
\end{tabular}
\end{minipage}
\end{center}
\end{table}

\subsubsection{DAiSEE}
On the DAiSEE dataset, we conduct experiments to investigate the performance of CavT, comparing with ResNet+TCN \cite{abedi2021improving}, C3D \cite{geng2019learning}, I3D \cite{zhang2019novel}, Swin-L \cite{liu2022video}, S3D \cite{xie2018rethinking}, and DFSTN \cite{liao2021deep}. We also evaluated the performance of BorS+CavT configured with sampling times $r$ = 3, comparing to CavT. Liao et el. does not publish code with their paper so we report the experimental results from \cite{liao2021deep}.

\begin{table}[h]
\begin{center}
\begin{minipage}{250pt}
\caption{Performance comparison on DAiSEE.}\label{tab:daisee-method-comp}%
\begin{tabular}{p{2cm}p{1.1cm}p{1.1cm}p{1.1cm}p{1.1cm}p{1.1cm}}
\toprule
Symbol & MSE & MMSE & Frames & FLOPs & Param\\
\midrule
CavT  &  0.0391  &  0.2012  &  8  &  25.03  &  119.85\\ 
BorS+CavT  &  \pmb{0.0377}  &  \pmb{0.1846}  &  8  &  25.03  &  119.85\\ 
ResNet+TCN  &  0.039  &  0.204  &  8  &  4.26  &  12.95\\ 
C3D  &  0.0401  &  0.229  &  64  &  115.79  &  63.21\\ 
I3D  &  0.0397  &  0.2145  &  64  &  111.33  &  14.39\\ 
Swin-L  &  0.0393  &  0.2308  &  32  &  545.1  &  195.58\\ 
S3D  &  0.0402  &  0.2172  &  64  &  72.69  &  10.02\\ 
DFSTN  &  0.0422  &  n/a  &  n/a  &  n/a  &  n/a\\ 
\botrule
\end{tabular}
\end{minipage}
\end{center}
\end{table}

Table~\ref{tab:daisee-method-comp} shows the experimental results of applying different methods to the engagement intensity prediction task on the DAiSEE dataset. The MSE and MMSE are shown in the second and third column, respectively. Table IV indicates that the BorS+CavT method achieves the lowest MSE (0.0377) and the lowest MMSE (0.1846) among all of methods. Compared with the experimental results conducted on the EmotiW-EP dataset, the BorS+CavT method does not significantly boost the performance of CavT (MSE = 0.0391). The reason behind is that DAiSEE videos are relatively short (only 10-second long), and there are only 300 frames in each video. It is difficult for BorS to generate heterogeneous video sequences. An interesting thing is that the MMSE values of all of methods are high. This indicates that class imbalance instead of insufficient data is a major obstacle to the prediction of engagement intensity. We leave it open for further research.

The fourth column of Table IV shows the number of frames in a video sequence. Compared to prior methods, we configure BorS+CavT with the least number of frames (8) and still attain the lowest MSE. FLOPs and Parameter are shown in the fifth and sixth column of this table. The computation cost and memory occupation of BorS+CavT, which is much cheaper than the state-of-the-art video transformer Swin-L, is identical to those of the CavT method.  

\subsection{Ablation study}
\subsubsection{Face extraction}
In MAGRU \cite{zhu2020multi} and DFSTN \cite{liao2021deep}, face extraction is a preliminary step intended for noise removal. This motivates an ablation study on the effect of face extraction. We performed face extraction by using OpenFace 2.0 \cite{baltrusaitis2018openface} on the EmotiW-EP dataset. Table~\ref{tab:ablation-face-extraction} indicates that CavT with face extraction outperforms CavT without face extraction in terms of MSE and MMSE. The results empirically prove that face is the most important region to express engagement intensity.

\begin{table}[h]
\begin{center}
\begin{minipage}{174pt}
\caption{Ablation study on the face extraction approach with CavT on EmotiW-EP.}\label{tab:ablation-face-extraction}%
\begin{tabular}{p{3cm}p{0.9cm}p{0.9cm}}
\toprule
 & MSE & MMSE \\
\midrule
w. face extraction  &  \pmb{0.0667}  &  \pmb{0.0806}\\ 
w/o face extraction  &  0.0759  &  0.1227\\ 
\botrule
\end{tabular}
\end{minipage}
\end{center}
\end{table}

\subsubsection{Weighted loss}
To address the class imbalance problem, MAGRU \cite{zhu2020multi} and DFSTN \cite{liao2021deep} adopted weighted loss function to pose higher weights on the minority class samples, i.e., samples labelled with very-low intensity level. We perform an ablation study over class weights of 1/3/6/9/12/15 for the very-low samples and report results in Table~\ref{tab:ablation-class-weight}. We find that weighted loss does not boost the performance of CavT. Instead, the weighted loss method makes MMSE higher. 

\begin{table}[h]
\begin{center}
\begin{minipage}{174pt}
\caption{Ablation study on posing class weights on the very-low samples with CavT on EmotiW-EP.}\label{tab:ablation-class-weight}%
\begin{tabular}{p{3cm}p{1cm}p{1cm}}
\toprule
Class Weight of DE  & MSE & MMSE \\
\midrule
1  &  \pmb{0.0667}  &  \pmb{0.0973}\\ 
3  &  0.0689  &  0.1115\\ 
6  &  0.0694  &  0.1191\\ 
9  &  0.0708  &  0.1056\\ 
12  &  \pmb{0.0667}  &  0.1105\\ 
15  &  0.0704  &  0.1239\\
\botrule
\end{tabular}
\end{minipage}
\end{center}
\end{table}

\subsubsection{Sampling times}
The sampling times of BorS $r$ is a hyper-parameter. Ablations of the sampling times are reported in Table~\ref{tab:ablation-sampling-times}. The MSE and MMSE of BorS+CavT gradually decrease with the rise of sampling times till BorS+CavT attains the lowest MSE at $r$ = 4. After that, BorS+CavT with too many sampling times yields the worse performance, i.e., MSE = 0.0604 when $r$ = 6. Theoretically, when BorS generates homogeneous video sequences with excessive sampling times, BorS+CavT suffers from the over-fitting problem with a high probability. Thus, we have to empirically do parameter search for r on different databases. For example, we configure BorS with $r$ = 4 on the EmotiW-EP dataset and $r$ = 3 on the DAiSEE dataset, respectively.

\begin{table}[h]
\begin{center}
\begin{minipage}{174pt}
\caption{Ablation study of sampling times with BorS+CavT on EmotiW-EP.}\label{tab:ablation-sampling-times}%
\begin{tabular}{p{3cm}p{1cm}p{1cm}}
\toprule
Sampling Times  & MSE & MMSE \\
\midrule
1  &  0.0667  &  0.0973\\ 
2  &  0.0634  &  0.0917\\ 
3  &  0.0548  &  0.0908\\ 
4  &  \pmb{0.0495}  &  \pmb{0.0796}\\ 
5  &  0.0549  &  0.0964\\ 
6  &  0.0604  &  0.0825 \\ 
\botrule
\end{tabular}
\end{minipage}
\end{center}
\end{table}

\subsubsection{Sampling alternative}
For the election of representatives of a slide window, we have two choices, random sampling or BorS. Unlike random sampling, BorS elects representatives located in the middle of a slide window or its descendant window in binary order, so that the representatives is distributed in an even pattern. As shown in Table~\ref{tab:ablation-sampling-alt}, compared to random sampling (MSE = 0.0652), BorS leads to superior performance (MSE = 0.0495).

\begin{table}[h]
\begin{center}
\begin{minipage}{174pt}
\caption{Ablation study on the two sampling methods with CavT on EmotiW-EP.}\label{tab:ablation-sampling-alt}%
\begin{tabular}{p{3cm}p{1cm}p{1cm}}
\toprule
Sampling method & MSE & MMSE \\
\midrule
BorS  &  \pmb{0.0495}  &  \pmb{0.0796}\\ 
random  &  0.0652  &  0.0893\\ 
\botrule
\end{tabular}
\end{minipage}
\end{center}
\end{table}

\section{Conclusions}
\label{sec:conclusion}
In this paper, we proposed an end-to-end architecture CavT for engagement intensity prediction that is based on self-attention between patches and class attention between class token and patches. To improve the performance of CavT, we further developed a binary order representatives sampling method BorS to address insufficient data for training a model and adaption of the class-attention transformer in image recognition for engagement intensity prediction. The proposed approaches (CavT and BorS+CavT) achieve the state-of-the-art performance on two widely used benchmarks, EmotiW-EP and DAiSEE. And BorS does improve the performance of CavT.

\section*{Acknowledgements}
\label{sec:ack}
This research was partially supported by the National Natural Science Foundation of China under grant No. 61876217, 62176175.

\section*{Declarations}
\begin{itemize}
\item National Natural Science Fundation of China [grants number 61876217, 62176175]
\item There are no conflict of interest/Competing interests.
\item The data and materials are available.
\item The code is available in the github repository, https://github.com/mountainai/cavt.
\end{itemize}


\bibliography{cavt}


\end{document}